\documentclass{article}
\usepackage{ijcai17}
\usepackage{times}
\usepackage{helvet}
\usepackage{courier}
\usepackage{pgf}
\usepackage{tikz}
\usetikzlibrary{arrows,automata}
\usepackage{verbatim}

\usepackage{amsmath}
\usepackage{graphicx}
\usepackage{xspace}
\usepackage[defblank]{paralist}
\usepackage{latexsym}

\newcommand{\nwsubseteq}{\rotatebox[origin=c]{45}{$\mathbf{\subseteq}$}}
\newcommand{\swsubseteq}{\rotatebox[origin=c]{-45}{$\mathbf{\subseteq}$}}

\newcommand{\qed}{$\Box$}

\newcommand{\ignore}[1]{}
\newcommand{\finish}[1]{}
\newcommand{\skipit}[1]{}

\newcommand{\calA}{{\cal A}}
\newcommand{\calC}{{\cal C}}
\newcommand{\calD}{{\cal D}}

\newcommand{\defeat}{\gg}

\newcommand{\NNP}{\textrm{Honest}$_P$}
\newcommand{\NNO}{\textrm{Honest}$_O$}
\newcommand{\NMP}{\textrm{Honest}$^{Min}_P$}
\newcommand{\NMO}{\textrm{Honest}$^{Min}_O$}

\renewcommand{\P}{\textrm{P}}
\newcommand{\NP}{\textrm{NP}}
\newcommand{\coNP}{\textrm{coNP}}
\newcommand{\PP}{\textrm{PP}}
\newcommand{\PSPACE}{\textrm{PSPACE}}

\newcommand{\ASPIC}{\textsc{ASPIC}}

\newcommand{\DefLog}{\textsc{DefLog}}

\title{Corruption and Audit in Strategic Argumentation}
\date{26 April, 2017}

\author{Michael J. Maher \\
Reasoning Research Institute, \\
Canberra, Australia    \\
E-mail: michael.maher@reasoning.org.au
}

\newcounter{clause}
\def\theclause{$c$\arabic{clause}}

{\end{tabbing}}

{\end{tabbing}}

\ignore{ }
\newtheorem{theorem}{Theorem}

\newtheorem{example}[theorem]{Example}
\newenvironment{proof}[1][Proof]{\begin{trivlist}\item[\hskip \labelsep {\bfseries #1}]}{\end{trivlist}}

\begin{document}

\newcommand{\Acom}{\calA_{Com}}
\newcommand{\AP}{\calA_P}
\newcommand{\AO}{\calA_O}
\newcommand{\Fcom}{F_{Com}}
\newcommand{\FP}{F_P}
\newcommand{\FO}{F_O}
 
\newcommand{\Rcom}{R_{Com}}
\newcommand{\RP}{R_P}
\newcommand{\RO}{R_O}
\newcommand{\DDL}{DL}

\maketitle

\begin{abstract}
Strategic argumentation provides a simple model of disputation and negotiation among agents.
Although agents might be expected to act in our best interests, there is little that enforces such behaviour.
(Maher, 2016) introduced a model of corruption and resistance to corruption within strategic argumentation.
In this paper we identify corrupt behaviours that are not detected in that formulation.
We strengthen the model to detect such behaviours, and show that, under the strengthened model, 
all the strategic aims in (Maher, 2016) are resistant to corruption.
\end{abstract}

\finish{
to address:

from reviews: definition of honest ... do i have 2 meanings, non-corrupt and non-random or arbitrary

check use of "rational"
}

\section{Introduction}

Strategic argumentation is an incomplete-knowledge game in which competing players
take turns in adding arguments to a common pool of arguments
such that at the end of a player's turn that player's strategic aim is (usually temporarily) achieved.
A player loses when she cannot successfully complete her turn.
Each player knows only her own arguments and the arguments in the common pool.

This gives a simple but insightful model of disputation and negotiation.
It is particularly suited as the basis for legal disputation between software agents.\footnote{
For example,
we can imagine agents with access to licensing agreements and copyright law
negotiating on take-down notices for claimed copyright infringement.
}
A wide variety of research has shown that defeasible rules,
which can be combined to create arguments,
are effective in representing contracts and legal reasoning \cite{modelling_precedents,contracts,Grosof04}.
However, there remains the question of trusting the software agents.

\cite{Maher16,Maher16b} investigates what occurs when agents/players corrupt the game
by violating the assumed privacy of a player's arguments.
Two cases are of corruption are considered:
\emph{espionage},
in which a player learns the arguments of her opponent, with her opponent unaware, and
\emph{collusion},
in which two players plan their play together in order to make a specific player win.
An instance of collusion is outlined in the following example.

\begin{example}    \label{ex:corrupt}
Consider the arguments in Figure \ref{fig:SAF8},
where vertices are arguments (grey if they can be played by $P$, white for $O$)
and edges are attacks of one argument on another.  Accepting argument A is $P$'s strategic aim.
Normal play would proceed as follows:
$P$ plays A, $O$ plays B1 (thus defeating A), $P$ plays C (restoring A by defeating B1), and $O$ plays D
(defeating C, and allowing B1 to defeat A).
Thus, normally, $P$ loses.

However, $P$ and $O$ might collude to ensure $P$ wins by playing as follows:
$P$ plays A, $O$ plays B1 \emph{and} B2, and $P$ plays C (restoring A).
$P$ now wins because $O$ has no effective move:
to play D would have no effect because it is defeated by B1.
This sequence of moves might also occur if $O$ committed espionage on $P$
in order to ensure $P$ wins.
\end{example}

\begin{figure}[h]
\vspace{-2.5cm}
\begin{center}
\includegraphics[width=0.5\textwidth]{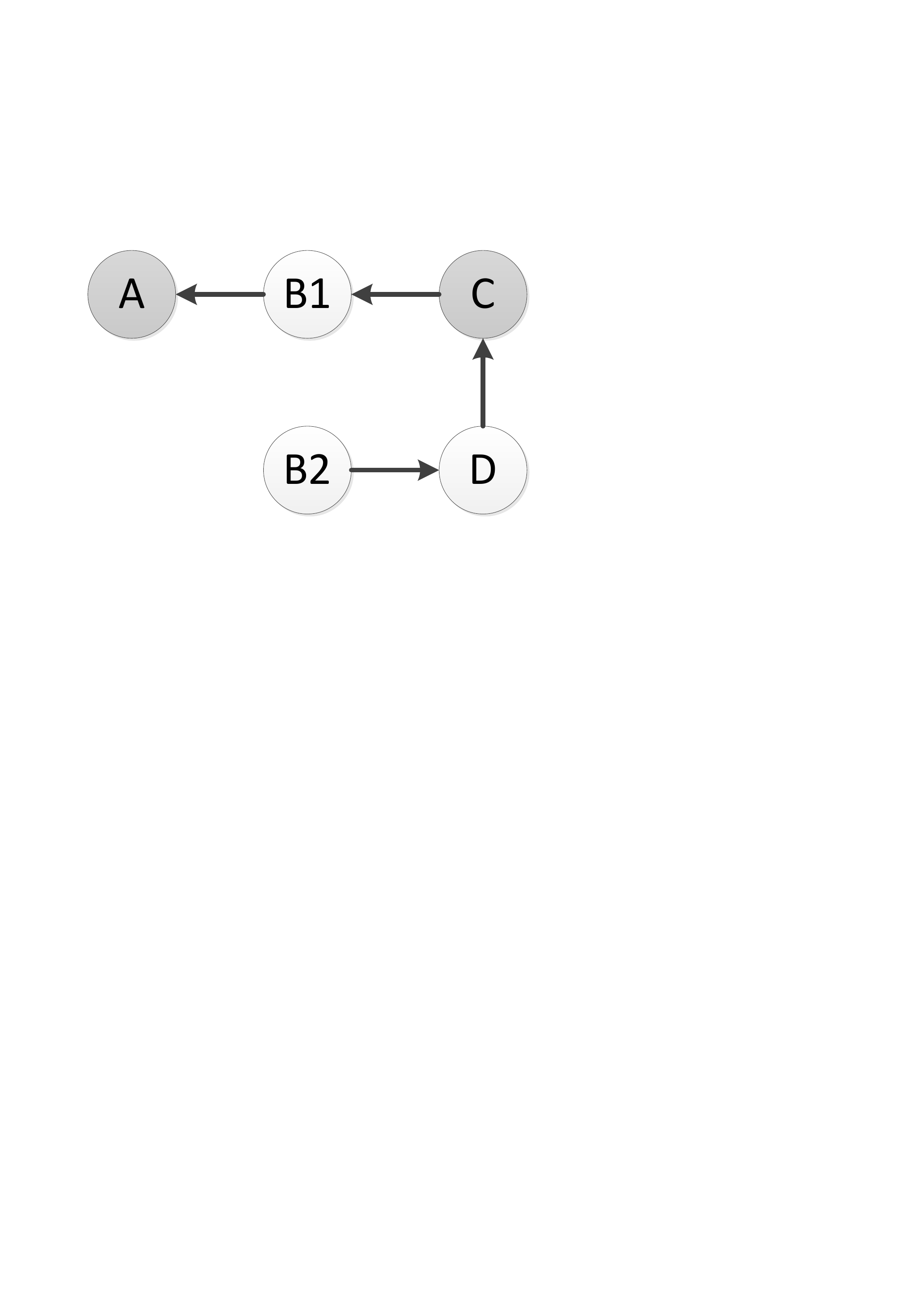}
\end{center}
\vspace{-7.8cm}
\caption{\label{fig:SAF8} A strategic argumentation game.}
\end{figure}

\begin{figure*}
\[
\begin{array}{rccccccccccccccccl}
&                    & \NP     &                        &                   &                      & \Sigma^p_2  \\
& \nwsubseteq &          & \swsubseteq  &                   & \nwsubseteq  &                      &  \swsubseteq  &                   & \\
\P &                    &          &                       & \Delta^p_2 &                       &                       &                        & \Delta^p_3 & \subseteq & \cdots &\subseteq & \NP^\PP  &  \subseteq  & \NP^{\NP^\PP}   & \cdots & \subseteq~~  \PSPACE\\
& \swsubseteq  &          & \nwsubseteq  &                   & \swsubseteq  &                      &  \nwsubseteq  &                   & \\
&                    & \coNP &                       &                    &                        &  \Pi^p_2 \\
\end{array}
\]
\caption{\label{fig:complexity} Some complexity classes in the polynomial counting hierarchy, ordered by containment.}
\end{figure*}

Clearly, in this example, the playing of B2 is foolish,
and might be considered a sign of incompetence or corruption.
However, the model of \cite{Maher16} permits this play.
Indeed, it might be argued that, in general, the playing of multiple arguments by $O$
further restricts the possible responses
of $P$, because those arguments might attack arguments that $P$ otherwise might have played.
On the other hand, the privacy of $O$'s arguments can be seen as a strategic advantage,
in which case playing multiple arguments is surrendering some of this advantage.

Although playing B2 ends badly for $O$ in this example,
we cannot consider it a sign of corruption if there are other circumstances in which playing B2
would improve the chances of $O$.
And even if playing B2 is uniformly bad, there is still the question of \emph{why}?
Is it because $O$ did not play the minimum number of arguments to achieve her aim
(as might be suggested from the surrendering strategic advantage line of thought),
or is it because $O$ did not play a minimal set of arguments that achieved her aim?
(In this example the two are the same, but in general they are different.)
Or is it because B2 attacks one of $O$'s own arguments, thus causing a self-inflicted injury?
Answers to these questions will enable us to insist on a stronger standard of behaviour from players,
which will make corruption harder to disguise.

In the following, we show that non-minimal sets of arguments are dominated by the corresponding minimal sets.
(Consequently, there is no circumstance in which playing B2 has a positive effect.)
By imposing the standard that players may only play minimal sets
we make it harder to perform collusion.
Furthermore, we show that this higher standard on players also improves resistance to collusion.


\section{Background}

\subsection{Abstract Argumentation}

This work is based on abstract argumentation in the sense of \cite{Dung95},
which addresses the evaluation of a static set of arguments.
An \emph{argumentation framework} $\calA = (S, \defeat)$ consists of a finite set of arguments $S$ and
a binary relation $\defeat$ over $S$, called the attack (sometimes, defeat) relation.
If $(a, b) \in \defeat$ we write $a \defeat b$ and say that $a$ attacks $b$.
We say there is a \emph{conflict} between arguments $a_1$ and $a_2$ if either $a_1 \defeat a_2$ or $a_2 \defeat a_1$.
A set of arguments $S' \subseteq S$ is \emph{conflict-free} if the restriction of $\defeat$ to $S'$ is empty.
The semantics of an argumentation framework is given in terms of \emph{extensions},
which are conflict-free subsets of $S$.

Given an argumentation framework, 
an argument $a$ is said to be \emph{accepted} in an extension $E$ if $a \in E$,
and said to be \emph{rejected} in $E$ if some $b \in E$ attacks $a$.
The set of rejected arguments in $E$ is denoted by $E^-$.
An argument that is neither accepted nor rejected in $E$ is said to be \emph{undecided} in $E$.
An argument $a$ is \emph{defended} by $E$ if every argument that attacks $a$ is attacked by some argument in $E$.
An extension $E$ of $\calA$ is \emph{complete} if it is conflict-free and, 
$a \in E$ iff $a$ is defended by $E$.
The least complete extension under the containment ordering exists and is called the \emph{grounded} extension.
It reflects a strongly sceptical attitude towards accepting arguments.
An extension $E$ of $\calA$ is \emph{stable} if it is conflict-free and for every argument $a \in S\backslash E$
there is an argument in $E$ that attacks $a$.


A semantics is defined to be a set of extensions:
the grounded semantics consists only of the grounded extension,
and
the stable semantics is the set of stable extensions.
There are many other semantics, including
the preferred, semi-stable, ideal, and eager semantics.
These are all formed as structurally-defined subsets of the set of complete extensions.
Such semantics will be called \emph{completist}.
There are also several other semantics that are not completist, notably the naive, stage, and CF2 semantics.
In this paper we focus on the grounded and stable semantics, 
since they are the semantics most commonly employed by existing languages.
However, some results extend more broadly, to all completist semantics.

Each semantics implicitly expresses a criterion for what arguments can coherently be accepted together,
given an argumentation framework.
Each extension in the semantics represents a ``reasonable'' adjudication, according to that criterion,
of the arguments in the argumentation framework.
Thus the grounded semantics is highly sceptical
while the stable semantics requires that no argument is left undecided.

Structural properties of an argumentation framework can influence the relationship between various semantics.
An argumentation framework is \emph{well-founded} 
if there is no infinite sequence of arguments $a_1, a_2, \ldots, a_i, a_{i+1}, \ldots$
such that, for each $i$, $a_{i+1}$ attacks $a_i$.
Such argumentation frameworks have a single complete extension,
which must be the grounded extension,
in which every argument is either accepted or rejected \cite{Dung95}.
Every completist semantics for such argumentation frameworks consists of this single extension.

\subsection{Computational Complexity}

We can view a complexity class as a set of decision problems.
We assume the reader has knowledge of the polynomial complexity hierarchy
(see, for example, \cite{catalog_complexity}).
We use $\P$ to refer to the class of problems solvable in polynomial time.
$\PSPACE$ is the class of decision problems solvable in polynomial space.
It contains the entire polynomial hierarchy $PH$.
As usual,
the notation $\calC^\calD$, where $\calC$ and $\calD$ are complexity classes,
refers to the class of problems that can be decided
by an algorithm of complexity $\calC$ with calls to a $\calD$ oracle.

To address counting aims in strategic argumentation we need to consider the complexity class $\PP$
and related classes.
Roughly, $\PP$ is the class of decision problems that have more accepting paths than rejecting paths
in a nondeterministic Turing machine.
A complete problem for $\PP$ is to decide whether a given Boolean formula is satisfied by
more than half of the assignments to its variables.
$\PP$ contains both NP and coNP, 
but there is an oracle relative to which $\Delta^p_2$ is not contained in $\PP$
\cite{Beigel94a}. 
$\NP^\PP$ and $\NP^{\NP^\PP}$ lie between the polynomial hierarchy and $\PSPACE$, 
that is, $PH \subseteq \NP^\PP \subseteq \NP^{\NP^\PP} \subseteq \PSPACE$.
The \emph{counting polynomial-time hierarchy} \cite{Wagner}
is the extension of the polynomial hierarchy that also involves the complexity class $\PP$.

Figure \ref{fig:complexity} shows some complexity classes in the counting polynomial hierarchy,
under the containment ordering.
The containments are believed by the majority of complexity theorists to be strict,
but the strictness remains open.
Complete problems of a complexity class are the hardest problems in that class.
Consequently, if Problem A is (say) in $\Delta^p_2$ and Problem B is $ \Sigma^p_2$-complete
then, in an informal sense, B is harder than A.

\section{Strategic Argumentation}   \label{sect:SA}

Strategic argumentation provides a simple model of dynamic argumentation.
Originally  \cite{stratarg07,SA} it was formulated for a concrete argumentation system based in a defeasible logic,
but we will use the model of \cite{Maher16} which is defined in terms of abstract argumentation.
In strategic abstract argumentation, players take turns to
add arguments to an argumentation framework.
At each turn, the player adds arguments so that the argumentation framework is in a desired state.
We refer to such states interchangeably as \emph{desired outcomes} or \emph{strategic aims} of the player.
A player loses the strategic argumentation game when she is unable to achieve her desired outcome.
In general, both players can win if the argumentation reaches a state that is desired by both players,
but in this paper we consider an adversarial setting where the players' aims are mutually exclusive.
We say that a player is \emph{honest} if she plays rationally with the aim of winning, 
with only information that is revealed by play of the game
(that is, she does not participate in corruption).

Strategic abstract argumentation is formalized as follows \cite{Maher16}.
We assume there are two players, a proponent $P$ and her opponent $O$.
A \emph{split argumentation framework} $(\Acom, \AP, \AO, \defeat)$
consists of three sets of arguments:
$\Acom$ the arguments that are common knowledge to $P$ and $O$,
$\AP$ the arguments available to $P$,
and
$\AO$ the arguments available to $O$;
and an attack relation $\defeat$ over $\Acom \cup \AP \cup \AO$.
$\AP$ is assumed to be unknown to $O$, and $\AO$ is unknown to $P$.
Each player is aware of $\defeat$ restricted to the arguments they know.
Thus, the privacy of a player's arguments is a strategic advantage.
We assume that $P$'s desired outcome
is that a distinguished argument $a$ (called the \emph{focal} argument) is accepted, in some sense,
while $O$'s aim is to prevent this.
Starting with $P$, the players take turns in adding sets of arguments to $\Acom$ from their available arguments,
ensuring that their desired outcome is a consequence of the resulting argumentation framework\footnote{
Each player's move is a normal expansion \cite{BaumannB10}.
}.
As play continues, the set of arguments that are common knowledge $\Acom$ becomes larger.
When a player is unable to achieve her aim on her turn to play, she loses.

We represent a split argumentation framework as a graph as follows:
each argument is a vertex, and there is a directed edge from $A$ to $B$ iff $A$ attacks $B$.
Arguments in $\AP$ are grey and arguments in $\AO$ are white.
Thus the split argumentation framework in Figure \ref{fig:SAF8}
is $( \emptyset, \{A, C\}, \{B1, B2, D\}, \defeat )$,
where $\defeat$ is the relation described by the edges in the figure.

It might be questioned whether it is necessary to allow players to add more than one argument.
Indeed, several persuasion games permit players to add only a single argument at a time \cite{VreesPrakken,PAF,ModgilC09}.
Such games have a different motivation: to give an operational characterization of argumentation semantics.
In our context, where the game determines the winner of a dispute,
it seems unfair that a player with a greater number of arguments is unable to take advantage
of that fact.
The following example shows that the restriction of moves to a single argument
can prevent a player from responding to her opponent, even though she has the arguments to win.

\begin{example}
Consider the split argumentation framework presented in Figure \ref{fig:SAFmulti},
where $P$'s desired outcome is that the argument A is accepted.
In any completist semantics, the corresponding argumentation framework supports A.
This follows from the fact that the argumentation framework is well-founded \cite{Dung95}.
$P$ can be seen to be in a dominant position to defend A when all the arguments are known,
since all $O$'s arguments that attack A are themselves attacked by arguments of $P$ (G and H)
that are not attacked.

Now consider the following sequence of moves:
$P$ plays A, $O$ responds with B, and then the arguments C, D and E are played in turn.
$O$ then plays F.
At this point, $O$'s arguments B and D are undefeated, since F attacks their attackers.
If $P$ is only allowed to play a single argument, then she loses.
On the other hand, if $P$ can play both G and H then she wins.

With hindsight, in this game, it appears that $P$ has made poor choices of which argument to play.
However, $P$ has no knowledge of $O$'s arguments, and so, to $P$, the choice between C and G to attack B
is entirely symmetric, as is the choice between E and H to attack D. 
\end{example}

\begin{figure}
\vspace{-0.5cm}
\begin{center}
\includegraphics[width=0.5\textwidth]{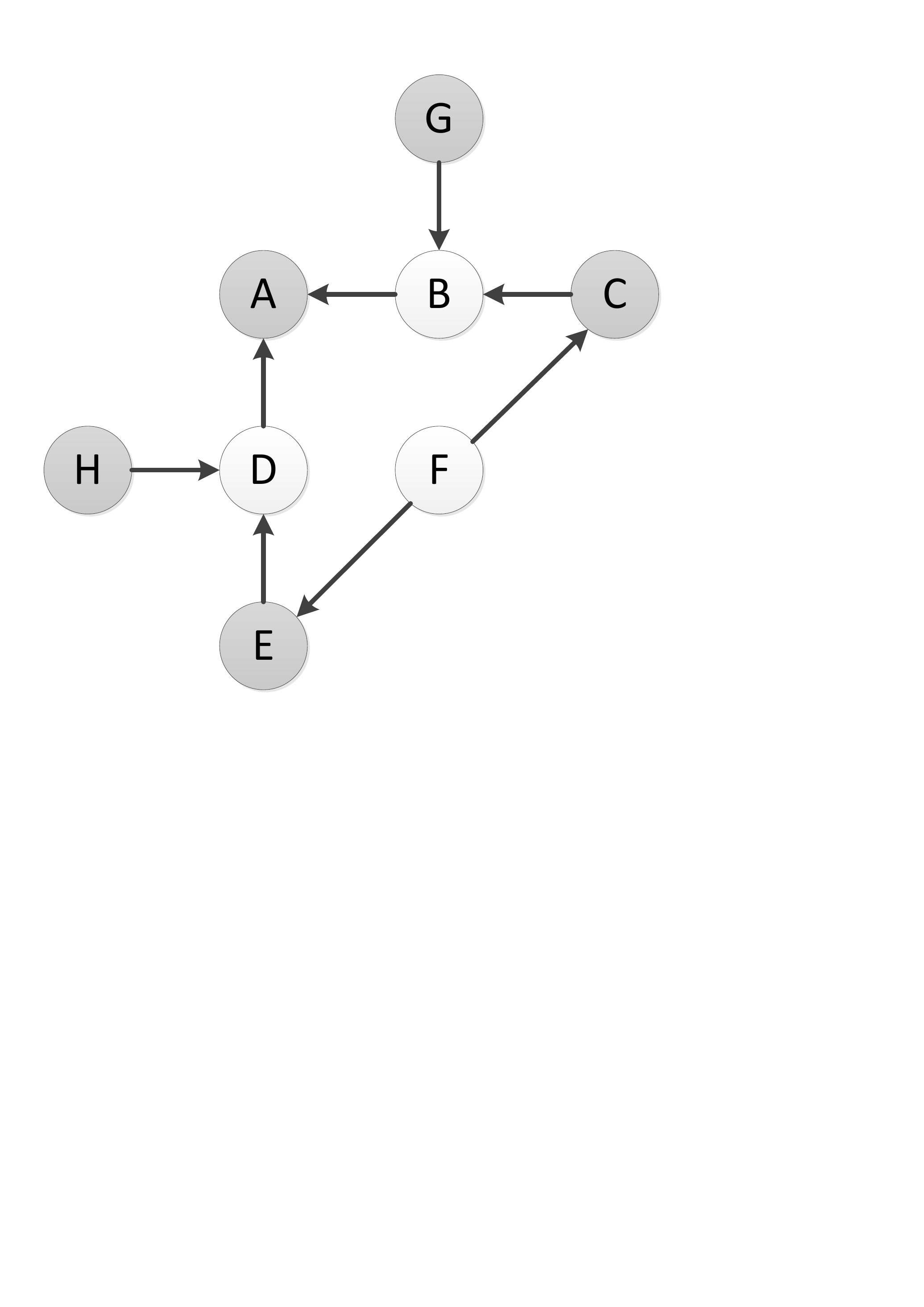}
\end{center}
\vspace{-5.8cm}
\caption{\label{fig:SAFmulti} A strategic argumentation game.}
\end{figure}

\cite{Maher16,Maher16b} identifies several plausible strategic aims that the proponent $P$ might have
under an argumentation semantics $\sigma$:

\begin{enumerate}
\item 
\textbf{Existential:}
$a$ is accepted in at least one $\sigma$-extension
\item 
\textbf{Universal:}
$a$ is accepted in all $\sigma$-extensions
\item 
\textbf{Unrejected:}
$a$ is not rejected in any $\sigma$-extension
\item 
\textbf{Uncontested:}
$a$ is accepted in at least one $\sigma$-extension and is not rejected in any $\sigma$-extension
\item 
\textbf{Plurality:}
$a$ is accepted in more $\sigma$-extensions than it is rejected
\item 
\textbf{Majority:}
$a$ is accepted in more $\sigma$-extensions than it is not accepted
\item 
\textbf{Supermajority:}
$a$ is accepted in at least twice as many $\sigma$-extensions than it is not accepted
\end{enumerate}

\begin{table*}[t]
\begin{center}
\begin{tabular}{|l||r|r|r|r|r|r|}
\hline
                                   & ~$AV_P$~  & $AV_O$~ &    ~~~~~$DO_P$~   & ~~~~~$DO_O$~  & \NNP~ & \NNO~ \\  
\hline \hline
Grounded semantics & ~in PTIME~  & ~in PTIME~ & NP-c~  & NP-c~ & $\Delta^p_2$-c~ & $\Delta^p_2$-c~  \\ 
\hline \hline
\multicolumn{7}{|l|}{Stable semantics}   \\
\cline{2-7}
~~~~~Existential                  & NP-c~  & coNP-c~ &    NP-c~               & $\Sigma^p_2$-c~  & $\Delta^p_2$-c~ & $\Delta^p_3$-c~  \\  
\cline{2-7}
~~~~~Universal                   & coNP-c~  & NP-c~   & $\Sigma^p_2$-c~   & NP-c~   & $\Delta^p_3$-c~ & $\Delta^p_2$-c~            \\ 
\cline{2-7}
~~~~~Unrejected                & coNP-c~  & NP-c~    & $\Sigma^p_2$-c~  & NP-c~   & $\Delta^p_3$-c~ & $\Delta^p_2$-c~    \\  
\cline{2-7}
~~~~~Uncontested               & coNP-c~  & NP-c~  & $\Sigma^p_2$-c~  & NP-c~   & $\Delta^p_3$-c~ & $\Delta^p_2$-c~              \\ 
\cline{2-7}
~~~~~Plurality/Majority~         & PP-c~ & PP-c~   & $\NP^\PP$-c~ & $\NP^\PP$-c~  & ~$\P^{\NP^\PP}$-c~  & ~$\P^{\NP^\PP}$-c~ \\
\cline{2-7}
~~~~~Supermajority            & PP-c~ & PP-c~   & $\NP^\PP$-c~ & $\NP^\PP$-c~  & ~$\P^{\NP^\PP}$-c~  & ~$\P^{\NP^\PP}$-c~ \\ 
\hline
\end{tabular}
\bigskip
\caption{\label{table:AVDO} Complexity of Aim Verification and Desired Outcome problems, and honest play, for $P$ and $O$ 
under the grounded and stable semantics, for selected aims (of $P$).}
\end{center}
\end{table*}

The existential and universal aims are credulous and sceptical acceptance, respectively.
In addition to the above aims for $P$, $O$ aims to ``spoil'' or prevent such aims from being achieved.
$O$'s aims are the negation of the above aims.
Because the grounded semantics consists of a single extension,
all these aims, except Unrejected, are identical under that semantics.
Furthermore, the Unrejected aim has the same complexity as the other aims
under the grounded semantics, so we will not distinguish the different aims under this semantics.
Under the stable semantics no argument is undecided.
Consequently, the Plurality and Majority aims are identical for the stable semantics.

The problem of verifying that an aim is satisfied by some state of strategic argumentation
is a fundamental part of each move in a game,
and of the exploitation of corrupt behaviour.
\ \\ \ \\
\noindent
\textbf{The Aim Verification Problem for $P$}

\textbf{Instance}
An argumentation framework $(\Acom, \defeat)$,
an argumentation semantics,
and an aim.

\textbf{Question}
Is the aim satisfied under the given semantics by the given argumentation framework?
\ \\

The Desired Outcome problem \cite{Maher16} is the problem that a player must solve at each step
of a strategic abstract argumentation game.
It involves identifying that the player has a legal move.
\ \\ \ \\
\noindent
\textbf{The Desired Outcome Problem for $P$}

\textbf{Instance}
A split argumentation framework $(\Acom, \AP, \AO, \defeat)$
and a desired outcome for $P$.

\textbf{Question} 
Is there a set $I \subseteq \AP$ such that
$P$'s desired outcome is achieved in
the argumentation framework $(\Acom \cup I, \defeat)$?
\ \\

It is not difficult to see that this problem can be solved by a non-deterministic algorithm
with an oracle for the Aim Verification problem.

We refer to these problems (and the corresponding problems for $O$) as $AV_P$ ($AV_O$) and $DO_P$ ($DO_O$).

Playing strategic argumentation involves solving the desired outcome problem at each turn.
We can formulate this as a deterministic polynomially bounded algorithm with an oracle for
the player's desired outcome problem.
Consequently, we can identify the complexity of playing strategic argumentation as $\P^{DO}$,
where $DO$ is the complexity of the desired outcome problem.

Combining results of \cite{Maher16,Maher16b}, for $P$, with new results for $O$
we can identify the complexity of these problems and, consequently,
the complexity of normal, honest play in a strategic argumentation game.
Honest play involves the solving of a polynomially-bounded sequence of
Desired Outcome problems.
We use the suffix ``-c'' to indicate that the problem is complete for the given class.

\begin{theorem}
The complexity of Aim Verification and Desired Outcome problems
and the complexity of honest play, for both $P$ and $O$,
are as stated in Table \ref{table:AVDO}.
\end{theorem}
\skipit{
\begin{proof}
The complexity of $AV_P$ and $DO_P$ comes from \cite{Maher16,Maher16b}.
By \cite{Maher16b}, the complexity of $AV_O$ is the complement of that of $AV_P$.
As observed earlier, the complexity of $DO_O$ is bounded above by $\NP^{AV_O}$.
Completeness for these results ????

The complexity of honest play is in $\P^{DO}$ for both $P$ and $O$,
as described in \cite{Maher16b}.
A complete problem for $\Delta^p_i$ is to determine whether the value of the last (i.e. least significant) variable in
the lexicographically greatest solution of the corresponding QBF is $1$.
This reduces to honest play
by constructing a split argumentation framework where there are $n$ copies of the quantified Boolean formula with $n$ variables
with constraints forcing a greater solution than previously found.
Player does binary search to find the lexicographically greatest solution.
Opposition move merely enables access to the next copy.

\end{proof}
}

\section{Corruption and Audit}   \label{sect:corrupt}

Although corruption takes place away from the strategic argumentation game,
it has no point unless it is exploited to alter the outcome of the game.
Thus the presence of corruption might be detected from the play of the game.
Conversely, corrupt players will attempt to disguise their exploitation of corruption as honest play.
The computational cost of disguising corruption is generally greater than
the computational cost of honest play, which is the source of the notion of
\emph{resistance to corruption} \cite{Maher16}.

However, in the model of  \cite{Maher16},
for some aims under the stable semantics there is no resistance to collusion.
That is, the cost of disguising corruption is no greater than the cost of honest play.
Thus there is no computational disincentive to collusion for those aims.
We seek to remedy this weakness.

It is helpful to view the verification of honest play as an audit.
There are two parts to an audit: standards that should be upheld,
and testing for/verification of compliance to those standards.
In the model of  \cite{Maher16} there is only one, implicit standard for play:
that a player may not abandon play while she can make a move.
However, we saw in Example \ref{ex:corrupt} that this standard is not sufficient:
obvious cases of collusion or espionage remain compliant with this standard.
Hence, we must normalize a higher standard of play so that compliance to
that standard is a better test of honest play.
However, a standard is only effective if it does not interfere with honest play.
That is, a player should never face a choice between following the standard and
improving her chances of winning.

Example \ref{ex:corrupt} and the discussion in the Introduction
suggested three possible restrictions on a player's moves
that would enforce higher standards.
Of the three, the avoidance of self-inflicted injury requires refinement, since self-injury cannot always be avoided,
and its suitability as a standard remains open.
So we consider the other two.
We say a move is \emph{effective} if the resulting argumentation framework satisfies the player's aim.
A \emph{minimal effective} (or, simply, \emph{minimal}) move by a player is an effective move such that no subset of that move is effective.  
A \emph{minimal cardinality} move is an effective move such that no move with fewer arguments is effective.

The minimum cardinality restriction forces the omission of moves that are needed to win,
as the following example shows.  Thus it is not suitable as a standard.

\begin{example}    \label{ex:SAF1}
Consider the strategic argumentation framework in Figure \ref{fig:SAF16},
and play that proceeds as follows:
$P$ plays A, $O$ plays B1 and B2, and $P$ plays C1 and C2.
At this stage $O$ must defeat both C1 and C2, and she has two alternatives:
(1) play E, which attacks both C1 and C2, or
(2) play both D1 and D2, each attacking one of the C arguments.
Clearly (1) is the minimum cardinality move.
However, $P$ then responds with F, and wins.
In (2), the play of F is insufficient for $P$, since B2 remains undefeated.
Hence $O$ wins.

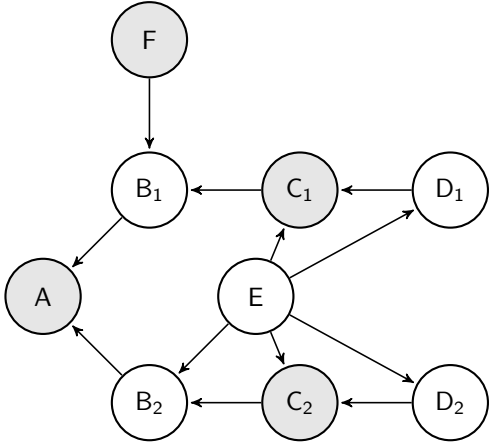
\begin{figure}
\begin{center}
\begin{tikzpicture}[->,>=stealth',shorten >=1pt,auto,node distance=2.0cm,
                    semithick]
  \tikzstyle{fake}=[]
  \tikzstyle{O}=[circle,fill=white,draw=black,text=black,thick,minimum size=10mm]
  \tikzstyle{P}=[circle,fill=black!10,draw=black,text=black,thick,minimum size=10mm]

   \node[P] (A)						{$\mathsf{A}$};
   \node[O] (B1)	[above right of=A]		{$\mathsf{B_1}$};
   \node[O] (B2)	[below right of=A]		{$\mathsf{B_2}$};
   \node[P] (C1)	[right of=B1]			{$\mathsf{C_1}$};
   \node[P] (C2)	[right of=B2]			{$\mathsf{C_2}$};
   \node[O] (D1)	[right of=C1]			{$\mathsf{D_1}$};  
   \node[O] (D2)	[right of=C2]			{$\mathsf{D_2}$};  
   \node[P] (F)	[above  of=B1]			{$\mathsf{F}$};
   \node[O] (E)	[above right of=B2]		{$\mathsf{E}$};

  \path (B1) 	edge				node {} (A)
  	   (B2) 	edge				node {} (A)
	   (C1)	edge	 			node {} (B1)
	   (C2)	edge	 			node {} (B2)
	   (D1)	edge	 			node {} (C1)
	   (D2)	edge	 			node {} (C2)
	   (F)		edge				node {} (B1)
	   (E)		edge				node {} (C1)
	   		edge				node {} (D1)
	   		edge				node {} (B2)
	   		edge				node {} (C2)
	   		edge				node {} (D2)
;

\end{tikzpicture}
\end{center}
\caption{\label{fig:SAF16} Split argumentation framework demonstrating non-dominance of minimum cardinality moves.}
\end{figure}

\end{example}

Thus,
we focus on avoiding the play of redundant arguments,
that is, the standard of playing minimal sets of arguments that achieve the players strategic aim.
As observed above, this standard is only reasonable if it does not eliminate an advantageous move for a player.
A move $m_1$ is \emph{dominated} by move $m_2$ if,
in every state of the game, and for every future sequence of moves,
$m_2$ leads to better or equal outcomes than $m_1$ for the player making the move.

\begin{theorem}
Every non-minimal move is dominated by a minimal move.

A player gains no advantage by making a non-minimal move,
and such moves can be disadvantageous
\end{theorem}
\begin{proof}
Consider the argumentation framework $\calA$ resulting after $P$ (say)
has made a minimal move $S$,
and the argumentation framework $\calA'$ resulting from the move $S \cup U$, where $U$ is non-empty
(the \emph{unnecessary} arguments).
Consider a potential move $M$ by the opponent $O$.
There are three cases:

1. $M$ is an effective move in $\calA$, but not in $\calA'$.
In this case, the presence of $U$ prevents $M$ in $\calA'$.
However, in response, in $\calA$, $P$ can play $U$ as her next move, neutralizing the effect of $M$.
This leaves the game in $\calA$ in a similar state as $\calA'$,
except that $O$'s move $M$ has been exposed.
$O$'s next move in $\calA$ is more constrained because elements of $M$ might attack
some of $O$'s remaining arguments.

Furthermore, in $\calA$, $P$ can play a minimal subset of $U$.

2. $M$ is an effective move in both $\calA$ and $\calA'$.
In this case, if $P$ has an effective move $M'$ in $\calA'$
then $M' \cup U$ is an effective move in $\calA$
and it leaves $\calA$ in the same state as $\calA'$.
On the other hand, that $P$ has an effective move $N$ in $\calA$
does not imply that she has an effective move in $\calA'$.
Example \ref{ex:corrupt} (and Figure \ref{fig:SAF8}) gives an example of a split argumentation framework
where a non-minimal move (by $O$, in that case) leads to her losing the game.

3. $M$ is an effective move in $\calA'$, but not in $\calA$.
In this case, $U$ has enabled a move by $O$ in $\calA'$ that is not effective in $\calA$.
\qed
\end{proof}

Thus, imposing the higher standard of minimal moves does not impact honest play.
There is no direct link between improved resistance to collusion and a higher standard of play,
but we will see that normalizing the play of minimal moves
does lead to an improvement in resistance to collusion.

There are several computational problems we must address in order to determine
whether or not resistance to corruption still holds under this higher standard.
The Minimality Problem is to verify that a given effective move is a minimal move.

\ \\ \ \\
\noindent
\textbf{The Minimality Problem for $P$}

\textbf{Instance}
A split argumentation framework $(\Acom, \AP, \AO, \defeat)$,
an argumentation semantics,
an aim for $P$,
and a move $M \subseteq \AP$ that achieves the aim for $P$.

\textbf{Question} 
Is $M$ a minimal set that achieves the aim under the given semantics?
That is, is there no subset $N \subset M$ such that
$P$'s desired outcome is achieved in
the argumentation framework $(\Acom \cup N, \defeat)$?
\ \\

It is clear that the complement of this problem can be solved by
a non-deterministic algorithm that guesses $N$ and uses an oracle for the Aim Verification problem.

The higher standard imposes some computational cost for honest play.
Each move must also verify that the proposed move is minimal.
However, it turns out that this additional work does not affect the computational complexity of honest play.
This is because the complement of the Minimality problem has the same complexity as the
Desired Outcome problem;
both do nondeterministic search using an oracle for Aim Verification.
The complexity of the Minimality problem and honest normal play are presented in
Table \ref{table:WSM}.

We now turn to the complexity of disguising corruption under this higher standard.
We first address the problem of disguising the exploitation of collusion to allow $P$ to win.
We begin by assuming that only $O$ is required to play minimal moves,
because it is the loser whose strategy must withstand scrutiny.

The two players must arrange an alternating sequence of moves by $P$ and $O$
such that each move is effective and $O$'s moves are minimal.
Furthermore, after $P$'s last move there must be no move for $O$
(that is, $O$ must lose).

\ \\
\noindent
\textbf{The Winning Sequence with Minimality Problem for $P$}

\textbf{Instance}
A split argumentation framework $(\Acom, \AP, \AO, \defeat)$
and a desired outcome for $P$.

\textbf{Question}
Is there a sequence of moves such that $P$ wins and $O$ always plays minimal moves?
\ \\

This problem can be solved by a non-deterministic algorithm
that guesses moves for $P$ and $O$
and uses oracles for the aim verification problem for $P$ and $O$,
and the minimality problem for $O$,
to ensure that the moves are valid.
Finally, an oracle for
the (complement of) the desired outcome problem for $O$
is needed to verify that $O$ loses.
We denote this problem by $WSM_P$, and the similar problem for $O$ is $WSM_O$. 
The variant where \emph{both} $P$ and $O$ must play minimal moves is denoted by $WSM'$.
This variant reflects a strict application of the standard to both players.

\begin{theorem}
Consider strategic argumentation under the standard that every move must be minimal.
The complexity of verifying that a move is minimal,
the complexity of normal, honest play,
and the complexity of Winning Sequence problems, for both $P$ and $O$,
are as stated in Table \ref{table:WSM}.
\end{theorem}
\skipit{
\begin{proof}
$Min_P$ is the complement of $DO_P$, at a high level.
The problem is clearly in $\coNP^{AV_P}$.
Completeness is obtained from the standard set-up \cite{Maher16}
by adding $\neg \psi$ and $\psi'$ as arguments where $\psi$ attacks $\neg \psi$, which attacks $\psi'$,
setting $\psi'$ as the focal argument,
and
arranging that if all $I_p$ and $I_{\neg p}$ are given then $\neg \psi$ is attacked (and $\psi'$ is accepted).
Consequently, the set of all $I_p$ and $I_{\neg p}$ is minimal iff $\psi$ is unsatisfiable.

Choosing a minimal effective move can be computed by guessing a move,
and then verifying that it is both effective and minimal.
Thus it can be computed by a non-deterministic algorithm with oracles for $AV$ and $Min$.
For $P$, this is in $\NP^{AV_P, Min_P}$; for the existential aim this is $\Sigma^p_2$,
for the universal, unrejected and uncontested aims this is $\Sigma^p_3$.
Thus $\NMP$ involves a polynomial sequence of choosing minimal effective moves.
Thus the complexity of $\NMP$ is $\Delta^p_3$ for existential aim and $\Delta^p_4$ for the others.

$WSM_P$ can be computed by a non-deterministic algorithm with oracles for $AV_P$, $AV_O$, $Min_O$, and (the complement of) $DO_O$.
Similarly for $WSM_O$ and $WSM'$.

Completeness ???
\end{proof}
}

\begin{table*}[t]
\begin{center}
\begin{tabular}{|l||r|r|r|r|r|r|r|}
\hline
                                   & ~$Min_P$~  & $Min_O$~ & \NMP~ & \NMO~ &   ~~$WSM_P$~   & ~~$WSM_O$~  & $WSM'$~ \\  
\hline \hline
Grounded semantics & ~coNP-c~  & ~coNP-c~ & $\Delta^p_2$-c~  & $\Delta^p_2$-c~& $\Sigma^p_2$-c~  & $\Sigma^p_2$-c~ & $\Sigma^p_2$-c~ \\ 
\hline \hline
\multicolumn{8}{|l|}{Stable semantics}   \\
\cline{2-8}
~~~~~Existential                  & coNP-c~        & $\Pi^p_2$-c~  & $\Delta^p_2$-c~  & $\Delta^p_3$-c~ &    $\Sigma^p_3$-c~  & $\Sigma^p_2$-c~ &    $\Sigma^p_3$-c~ \\  
\cline{2-8}
~~~~~Universal                   & $\Pi^p_2$-c~  & coNP-c~  & $\Delta^p_3$-c~  & $\Delta^p_2$-c~ & $\Sigma^p_2$-c~   & $\Sigma^p_3$-c~ &    $\Sigma^p_3$-c~ \\ 
\cline{2-8}
~~~~~Unrejected                & $\Pi^p_2$-c~  & coNP-c~   & $\Delta^p_3$-c~  & $\Delta^p_2$-c~ & $\Sigma^p_2$-c~  & $\Sigma^p_3$-c~ &    $\Sigma^p_3$-c~ \\  
\cline{2-8}
~~~~~Uncontested               & $\Pi^p_2$-c~  & coNP-c~  & $\Delta^p_3$-c~  & $\Delta^p_2$-c~ & $\Sigma^p_2$-c~  & $\Sigma^p_3$-c~ &    $\Sigma^p_3$-c~ \\ 
\cline{2-8}
~~~~~Plurality/Majority~         & co$\NP^\PP$-c~ & co$\NP^\PP$-c~  & $\P^{\NP^\PP}$-c~ & $\P^{\NP^\PP}$-c~ & $\NP^{\NP^\PP}$-c~ & $\NP^{\NP^\PP}$-c~ & $\NP^{\NP^\PP}$-c~   \\
\cline{2-8}
~~~~~Supermajority            & ~co$\NP^\PP$-c~ & ~co$\NP^\PP$-c~  & $\P^{\NP^\PP}$-c~ & $\P^{\NP^\PP}$-c~  & ~$\NP^{\NP^\PP}$-c~  & ~$\NP^{\NP^\PP}$-c~  & ~$\NP^{\NP^\PP}$-c~  \\ 
\hline
\end{tabular}
\bigskip
\caption{\label{table:WSM} Complexity of Minimality problems, normal play with the minimality standard, and Winning Sequence problems for $P$ and $O$ 
under the grounded and stable semantics, for selected aims (of $P$).}
\end{center}
\end{table*}

In the case of espionage, one player, say $P$, illicitly learns her opponent's arguments $\AO$ and desires a strategy
that will ensure $P$ wins, no matter what moves $O$ makes.
A \emph{strategy} for $P$ in a split argumentation framework $(\Acom, \AP, \AO, \defeat)$
is a function from a set of common arguments 
to the set of arguments to be played in the next move.
A sequence of moves $S_1, T_1, S_2, T_2, \ldots$
resulting in common arguments $\Acom^{P, 1}, \Acom^{O, 1}, \Acom^{P, 2}, \Acom^{O, 2}, \ldots$
is \emph{consistent with} a strategy $s$ for $P$ if,
for every $j$,
$S_{j+1} = s(\Acom^{O, j}, \AP)$.
A strategy for $P$ is \emph{winning} if
every valid sequence of moves consistent with the strategy is won by $P$.

\ \\
\noindent
\textbf{The Winning Strategy with Minimality Problem for $P$}

\textbf{Instance}
A split argumentation framework $(\Acom, \AP, \AO, \defeat)$
and a desired outcome for $P$.

\textbf{Question}
Is there a winning strategy for $P$ that only makes minimal moves?
\ \\

Under the higher standard we impose,
strategic argumentation is  \emph{resistant to collusion (espionage)} if
the complexity of the Winning Sequence (Winning Strategy) with Minimality problem is greater than
the complexity of honestly playing the strategic argumentation game,
under the widely-believed complexity-theoretic assumption that the polynomial hierarchy does not collapse.
In that case, the computational work needed to exploit the corrupt behaviour
is greater than that required to simply play the argumentation game.

\begin{theorem}
Under the grounded and stable semantics,
and for every desired outcome under consideration,
the Winning Sequence Problem with Minimality 
has greater complexity than normal, honest play.

Consequently, in such situations, strategic argumentation is resistant to collusion. 
\end{theorem}
\skipit{
\begin{proof}
\end{proof}
}

Thus, we see that the higher standard of behaviour required of players has the desirable side-effect
that all aims under the stable semantics are resistant to collusion.

The restriction to minimal moves does not affect resistance to espionage.

\begin{theorem}
Under every completist semantics,
and for every desired outcome under consideration,
the Winning Strategy Problem with Minimality is $\PSPACE$-complete.

Consequently, in such situations, strategic argumentation is resistant to espionage. 
\end{theorem}
\skipit{
\begin{proof}
By Theorem \ref{thm:dom}, $P$ has a winning strategy iff $P$ has a winning strategy that makes only minimal moves.
By Theorems 5 and 13 of \cite{Maher16}, this problem is $\PSPACE$-complete
for both grounded and stable semantics.
The demonstration of $\PSPACE$-hardness uses an argumentation framework that is well-founded.
Consequently, all completist semantics consist only of the grounded extension,
and hence the problem, for any completist semantics, is $\PSPACE$-complete.
However, the complexity of playing the game honestly, as given in Table \ref{table:AVDO}
is lower in the polynomial counting hierarchy.
\end{proof}
}

\section{Beyond Resistance?}
Rather than rely on resistance to collusion,
we might want to impose even higher standards, 
so that collusion cannot be disguised.
We briefly explore this possibility.

The first question is:
Is the current high standard sufficient to prevent the disguise of collusion?
Unfortunately, the answer is no, as the following example shows.

\begin{figure}

\begin{center}
\begin{tikzpicture}[->,>=stealth',shorten >=1pt,auto,node distance=2.0cm,
                    semithick]
  \tikzstyle{fake}=[]
  \tikzstyle{O}=[circle,fill=white,draw=black,text=black,thick,minimum size=10mm]
  \tikzstyle{P}=[circle,fill=black!10,draw=black,text=black,thick,minimum size=10mm]
  
   \node[P] (A)							{$A$};
   \node[O] (B)	[right of=A]				{$B$};
   \node[P] (C)	[right of=B]				{$C$};
   \node[O] (D)	[right of=C]				{$D$};
   \node[O] (E)	[below of=A]				{$E$};
   \node[P] (F)	[below  of=B]				{$F$};
   \node[P] (G)	[below of=E]				{$G$};
   \node[P] (H)	[below  of=F]				{$H$};

  \path (B) 	edge				node {} (A)
	   (C)	edge	 			node {} (B)
	   (D)	edge				node {} (C)
	   (E)	edge				node {} (A)
	   (F)	edge				node {} (B)
	   (G)	edge				node {} (E)  
	   (H)	edge				node {} (E)
	   	edge				node {} (F)
;
\end{tikzpicture}
\end{center}

\caption{\label{fig:AH}A strategic argumentation game}
\end{figure}
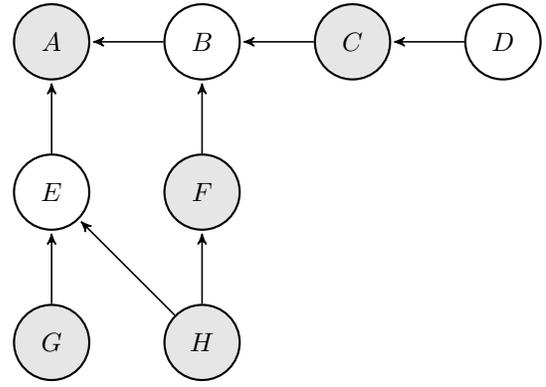

\begin{example}
Consider the split argument framework depicted in Figure \ref{fig:AH},
where arguments in $\AP$ are grey and arguments in $\AO$ are white,
and A is the critical argument.
If $P$ desists from playing $H$ then $P$ will win,
since the two arguments attacking $A$ ($B$ and $C$) can be attacked by $P$'s arguments $F$ and $G$, which cannot be attacked by $O$.
For example, the sequence of moves: $A, B, F, E, G$ results in $P$ winning.

On the other hand, the sequence of moves: $A, E, H, B, C, D$ results in $O$ winning.
Thus, $P$ and $O$ can collude to ensure $O$ wins.
\end{example}

This example suggests that a variation of the avoidance of self-inflicted injury might be needed to detect collusion more thoroughly.
Which leads to a second question: is it possible to impose a high enough standard that any collusion
cannot be disguised as compliant play?
Again the answer is no.

Consider the argumentation game in Figure \ref{fig:choice} under the grounded semantics, where $A$ is the critical argument.
After $P$ plays $A$, $O$ has the choice of playing $B$ or $C$.
Depending on this choice, either $P$ or $O$ will win.
If $P$ and $O$ collude they can determine the outcome,
but any real restriction imposed by a standard will restrict to one possible outcome,
so it cannot be a justified standard.
Thus any collusion in this game cannot be detected by imposing higher standards,
assuming standards are required not to interfere with honest (i.e. non-corrupt) play.

Hence, we see that collusion cannot be prevented simply by imposing higher standards.
Indeed, it appears that collusion is detectable iff the argumentation game with standards is fixed:
there is only one possible winner.
We must continue to rely on computational difficulty to discourage corruption.

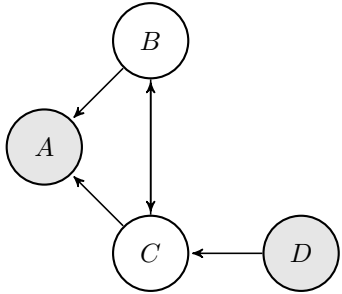
\begin{figure}[t]

\begin{center}
\begin{tikzpicture}[->,>=stealth',shorten >=1pt,auto,node distance=2.0cm,semithick]

  \tikzstyle{O}=[circle,fill=white,draw=black,text=black,thick,minimum size=10mm]
  \tikzstyle{P}=[circle,fill=black!10,draw=black,text=black,thick,minimum size=10mm]
  
   \node[P] (A)							{$A$};
   \node[O] (B)	[above right of=A]			{$B$};
   \node[O] (C)	[below right of=A]			{$C$};
   \node[P] (D)	[right of=C]				{$D$};  

  \path (B) 	edge				node {} (A)
	   	edge	 			node {} (C)
	   (C)	edge	 			node {} (A)
	   	edge	 			node {} (B)
	   (D)	edge				node {} (C)
;
\end{tikzpicture}
\end{center}

\caption{\label{fig:choice}A strategic argumentation game}
\end{figure}

\section{Defeasible Rule Languages}

The results above for abstract strategic argumentation extend to concrete argumentation languages based on defeasible rules.
This includes structured argumentation systems such as $\ASPIC$ \cite{ASPIC} and its derivatives,
and assumption-based argumentation (ABA) \cite{Bondarenko}, under grounded or stable semantics.
But it also includes a range of other languages such as defeasible logics \cite{TOCL10,MN10},
Ordered Logic \cite{OrderedLogic},
$LPDA$ \cite{LPDA} and Rulelog \cite{Rulelog}, among many that correspond to the grounded semantics.
Similarly, several concrete languages correspond to the stable semantics:
defeasible logics under stable model semantics \cite{flexf,Maier13} ,
$\DefLog$ \cite{DefLog},
and $ASPDA$ \cite{ASPDA}, among others.
The emulation of abstract argumentation by these languages \cite{cdr_ICLP}
is the key element in carrying complexity results for abstract strategic argumentation
to these languages.

\section{Conclusion}

Raising the standard of play to require minimal moves makes exploitation of collusion more difficult
because it eliminates an easy way to introduce arguments into the common pool.
Surprisingly, however, we found that it also improved resistance to collusion.
Such behaviour should not be expected, in general, from the raising of standards;
it will depend, in part, on the complexity of verifying the standard.

We briefly touched on the possibility of even higher standards:
requiring a minimum cardinality move,
or a move that avoids self-inflicted injury.
It remains an interesting open problem whether the minimality standard
can be strengthened along those lines,
and how that might affect resistance to collusion.

\bibliographystyle{named}
\bibliography{audit_stripped}


\end{document}